\documentclass{article}

     \PassOptionsToPackage{round, sort&compress, authoryear}{natbib}
\usepackage{arxiv}

\usepackage[utf8]{inputenc} % allow utf-8 input
\usepackage[T1]{fontenc}    % use 8-bit T1 fonts
\usepackage{hyperref}       % hyperlinks
\usepackage{url}            % simple URL typesetting
\usepackage{booktabs}       % professional-quality tables
\usepackage{amsfonts}       % blackboard math symbols
\usepackage{nicefrac}       % compact symbols for 1/2, etc.
\usepackage{microtype}      % microtypography

\usepackage{times,amssymb,bm}
\usepackage{amsmath}
\usepackage{graphicx, float, wrapfig, xcolor}
\usepackage{smile}

\def\KSD{{\rm KSD}}
\def\IPM{{\rm IPM}}
\def\F{{\rm F}}

\title{Stein Neural Sampler}

\author{%
	Tianyang Hu\thanks{These authors contribute equally to this work.} %\ \thanks{Purdue University}
	\\
	Purdue University \\
 	\And
 	Zixiang Chen$^*$ \\
 	UCLA\\
 	\And
 	Hanxi Sun$^*$ \\
	Purdue University \\
	\And
	Jincheng Bai\\
	Purdue University \\
 	\And
 	Mao Ye\\
 	UT Austin\\
 	\And
 	Guang Cheng\\
	Purdue University \\
}

\begin{document}
\maketitle

\begin{abstract}
We propose two novel samplers to generate high-quality samples from a given (un-normalized) probability density. Motivated by the success of generative adversarial networks, we construct our samplers using deep neural networks that transform a reference distribution to the target distribution. Training schemes are developed to minimize two variations of the Stein discrepancy, which is designed to work with un-normalized densities. Once trained, our samplers are able to generate samples instantaneously. We show that the proposed methods are theoretically sound and experience fewer convergence issues compared with traditional sampling approaches according to our empirical studies.
\end{abstract}

%%%%%%%%%%%%%%%%%%%%%%%%%%%%%%%%%%%%%%%%%%%%%%%%%%%%%%%%%%%%%%%%%%%%%%%%%%%%%%%%%%%%%%%%%%%%%%%%%%%%%%
%%%%%%%%%%%%%%%%%%%%%%%%%%%%%%%%%%%%%%%%%%%%%%%%%%%%%%%%%%%%%%%%%%%%%%%%%%%%%%%%%%%%%%%%%%%%%%%%%%%%%%

\section{Introduction} 
A core problem in machine learning and Bayesian statistics is to approximate a complex target distribution given its probability density up to an (unknown) normalizing constant.
Take posterior sampling as an example, the target distribution is proportional to the prior times the likelihood and evaluation is hard without the normalizing constant. 
For decades, researchers have been relying on mainly Markov Chain Monte Carlo (MCMC) \citep{gamerman2006markov} and Variational Bayes (VB) \citep{kingma2013auto, blei2017variational} to evaluate such densities. However, MCMC can be slow to mix and hard to scale to large data sets or complex models. While VB is more computationally feasible, its performance is often hindered by the lack of capacity in the variational family and using Kullback-Leibler (KL) divergence as the training objective \citep{yao18yes}. 

Considering the weaknesses and advantages of MCMC and VB, it is desirable to construct a sampler with larger capacity, better objective and scales for modern machine learning tasks. 
%consistent in theory 
%and scales better with less convergence issue in practice.
%and able to generate sample instantaneously after being well trained.
\citet{villani2008optimal} showed that between any two non-atomic distributions, there always exists a measurable transformation. 
Therefore, we propose to learn such a transformation from an easy-to-sample reference distribution to the target distribution by modeling it within a sufficiently rich family of functions, such as neural networks \citep{raghu2016expressive}. 
%Therefore, a natural route to achieve the goal is to learn a preservable transformation $T$ from an easy-to-sample reference distribution $p_z(z)$ to the target distribution $q(x)$. 
%One way to learn such a transformation is to model it within a sufficiently rich family of functions such as neural networks \citep{raghu2016expressive}. 
The expressive power of deep neural networks in modeling complex distributions has been demonstrated by the recent success of Generative Adversarial Network (GAN), where the adversarial game between the generator and the discriminator enables a data-dependent, problem-specific objective that yields astonishing empirical performances never seen before \citep{goodfellow2014, radford2015unsupervised, brock2018large}.

Although both GAN and our sampler aim at generating samples from complicated distributions, GAN learns from a set of true samples (images), while our sampler is trained with the un-normalized true density $q$.
% Although having the true density seems powerful, it still remains difficult to evaluate the loss  complicated form prohibits us from learning 
An explicit form of $q$ seems more informative, however, it may involve intractable integrations when measuring the distance between the samples and the target. 
%the intractable normalizing constant brings difficulties, it is difficult to find a proper loss that correctly measures the distance between the sampling distribution and the target. 
To bypasses these difficulties, we turn to Stein discrepancy, which can serve as a measurement of sample quality.

\begin{figure}[ht]
    \label{fig:parallel}
    \centering
    \includegraphics[width=.92\textwidth]{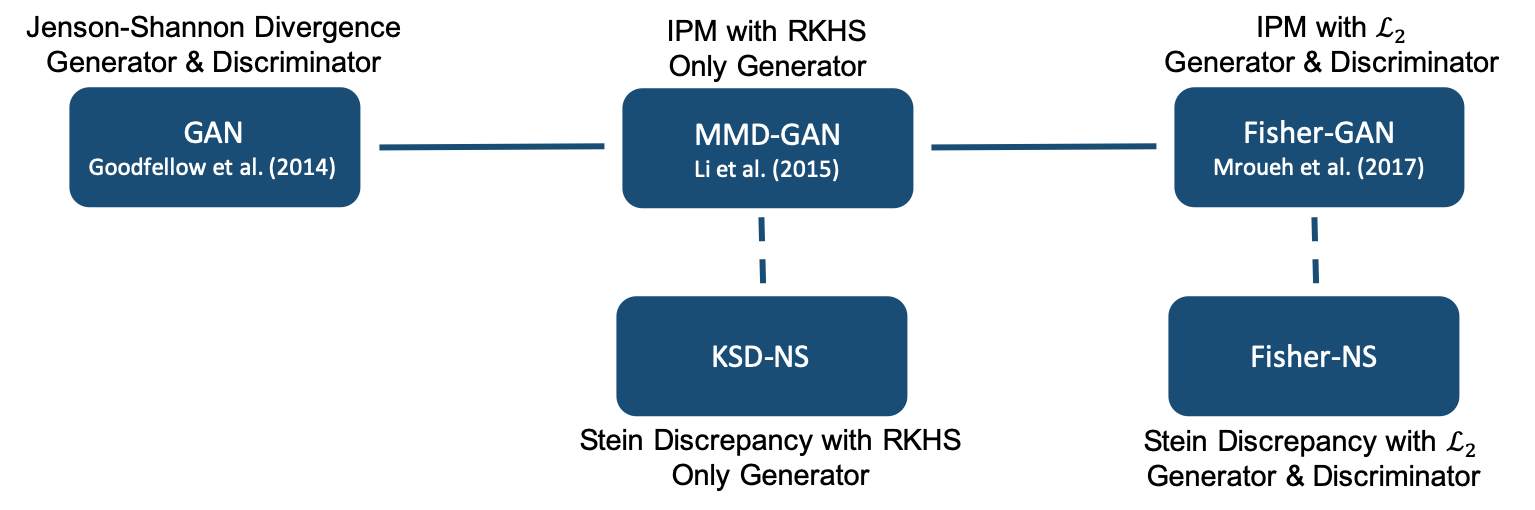}
    \caption{Summary of our proposed approaches and the relationship to existing GAN models. The training objectives and networks to train for each model are labeled in the figure. }
\end{figure}

%%%%%%%%%%%%%%%%%%%%%%%%%%%%%%%%%%%%%%%%%%%%%%%%%%%%%%%%
In this paper, we propose two novel sampling schemes based on Stein discrepancy that can directly learn preservable transformations constructed by neural networks: Kernelized Stein Discrepancy Neural Sampler (KSD-NS) and Fisher Divergence Neural Sampler (Fisher-NS).
The main contribution and advantages of our proposed sampling methods are:
\begin{itemize}
\item
Deep neural network is used to represent the transformation. Once trained, \textbf{independent} samples can be generated \textbf{instantaneously} from forward passes of the generator.
\item
Training is based on Stein discrepancy, which resembles the objective of GAN.
Different discriminative function spaces in Stein discrepancy is investigated. KSD-NS utilizes the unit ball in a reproducing kernel Hilbert space (RKHS) and is easier to train with theoretical guarantee. Fisher-NS enlarges the RKHS to $L_2$ space to have higher potentials. 
\item
Empirical studies show that our neural samplers perform well in both toy examples and real data.  
KSD-NS is more stable and Fisher-NS tends to achieve better results in higher dimensions.
\end{itemize}
The paper is organized as follows. Section \ref{background} introduces some necessary notions for our method. The proposed samplers, KSD-NS and Fisher-NS, are discussed in details in section 3 and 4. Related work is reviewed in Section 5 and experiments results are presented in Section 6. All proof of the theorems, along with more discussions about the methodology and the experiment setting can be found in the supplementary material. 

\section{Background}%BACKGROUND}
\label{background}
\paragraph{Stein's Identity}
Let $q(x)$ be a continuously differentiable density supported on $\cX\subseteq\RR^d$ and $\boldsymbol{f}(x)=[f_1(x), \cdots, f_d(x)]^\top$ be a smooth vector function satisfying some mild boundary conditions. Then, the Stein's identity states that
\begin{align}
    \label{eqn:steq1}
    \EE_{x\sim q}\sbr{S_q(x) \boldsymbol{f}(x)^\top  + \nabla_x \boldsymbol{f}(x)} = \bf{0},  
\end{align}
where $S_q(x) = \nabla_x \log q (x)$ is the score function of $q(x)$. Note that calculating $S_q (x)$ does not require the normalization constant in $q(x)$, which is often intractable in practice. This property makes Stein's identity an ideal tool for handling un-normalized target distributions.

\paragraph{Stein Discrepancy}
Let $p(x)$ be another smooth density supported on $\cX$. %if we replace the expectation $\EE_q[\cdot]$ in (\ref{eqn:steq1}) with $\EE_p[\cdot]$, 
If in (\ref{eqn:steq1}), the expectation is taking with respect to $p$ instead, the equality will not hold in general. This property naturally induces a distance between the two densities $p(x)$ and $q(x)$ by optimizing the right hand side of (\ref{eqn:steq1}) over all functions $\boldsymbol{f}$ within a function space $\cF$ \citep{gorham2015measuring}, 
\begin{align}
    \label{eqn:ksdexp}
    & {\cD(p, q, \cF)} = \sup_{\boldsymbol{f} \in \cF} \cbr{\EE_{x\sim p} \tr\rbr{S_q(x) \boldsymbol{f}(x)^\top  + \nabla_x \boldsymbol{f}(x)}},  
\end{align}
where $\tr(A)$ is the trace of matrix $A$. If $\cF$ is large enough, ${\cD(p, q, \cF)}=0$ if and only if $p=q$. However, $\cF$ cannot be too large. Otherwise, ${\mathcal D(p, q, \cF)}=\infty$ for any $p\ne q$. 

\paragraph{Kernelized Stein Discrepancy}
Let $\cH^d$ be a RKHS associated with kernel function $k(\cdot, \cdot)$. \citet{liu2016kernelized} showed that if the function space $\cF$ is the unit ball in $\cH^d$, the supremum in (\ref{eqn:ksdexp}) has a closed form solution, kernelized Stein discrepancy (KSD). $\KSD (p, q) = \EE_{x, x' \sim p} \sbr{u_q(x, x')}$, where
\begin{equation}
    \label{eqn:u_q}
    \begin{aligned}
         u_q(x, x') =~ & S_q(x)^\top k(x, x')S_q(x') + S_q(x)^\top \nabla_x k(x, x') 
         + \nabla_x k(x, x')^\top S_q(x')  \\
         & + \tr\rbr{\nabla_{x, x'}k(x, x')}.
    \end{aligned}
\end{equation}
The corresponding optimal discriminative function $\boldsymbol{f}^*$ satisfies $\norm{\boldsymbol{f}^*}_{\cH^d}=1$ and
\[
    \boldsymbol{f}^*(\cdot)\propto \EE_{x \sim p}\sbr{S_q(x) k(x, \cdot)  + \nabla_{x} k(x, \cdot)},
\]
%is shown to be $\boldsymbol{f}^*(x) /\norm{\boldsymbol{f}^*}_{\cH^d}$, where
%\[
%    \boldsymbol{f} ^*(\cdot)= \EE_{x \sim p}\sbr{S_q(x) k(x, \cdot)  + \nabla_{x} k(x, \cdot)},
%\]
%and $k(\cdot,\cdot)$ is the kernel associated with the RKHS and $\norm{\cdot}_{\mathcal{H}^d}$ is the corresponding norm. Plugging in the optimal $\boldsymbol{f}$ gives a closed form for $\cD(p, ~q)$ and its square is usually called Kernelized Stein Discrepancy (KSD) which can be written as

Empirical KSD measures the goodness-of-fit of samples $X=\{x_1, \cdots, x_n\}$ to a density $q(x)$. The minimum variance unbiased estimator can be written as 
\begin{align}
    \label{eqn:KSD_estimator}
    \hat{\KSD}(p, q) &= \frac{1}{n(n-1)} \sum_{i=1}^n\sum_{j \ne i}^n [u_q(x_i, x_j)]
\end{align}
Despite the ease of computation, RKHS is relatively small and may fail to detect non-convergence in higher dimensions \citep{gorham2017measuring}. 

\paragraph{GAN and Integral Probability Metrics (IPM)} 
GAN also learns to transform random noises to high-quality samples. 
%Its objective is to implicitly capture the underlying distribution of given samples by building a generator to sample from it. 
%GAN learns to produces similar samples to observations using neural networks. 
%In GANs, the generator is trained adversarially with another discriminator network. 
%The discriminator takes both true samples and generated samples as input and essentially conducts two sample test between them. 
The min-max game between the generator and discriminator networks optimally corresponds to minimizing the Jenson-Shannon divergence in the vanilla GAN \citep{goodfellow2014}. Other choices of divergence lead to variants of GAN such as Maximum Mean Discrepancy (MMD) \citep{li2015}, Wasserstein distance \citep{arjovsky2017wasserstein}, Chi-squared distance \citep{mroueh2017fisher}, etc. see an overview by  \citet{mroueh2017sobolev}. The aforementioned distances can all be seen as examples of IPM \citep{muller1997integral}, which measures the distance between two distributions $p$ and $q$ via the largest discrepancy in expectation over a class of well-behaved witness functions $\cF$:
\begin{align}
    \label{eqn:ipm}
    \IPM\rbr{p, q, \cF} =\sup_{f\in\mathcal F} \cbr{\EE_{x\sim p}\sbr{f(x)} - \EE_{x\sim q}\sbr{f(x)}}. 
\end{align}
%Similar as in Stein discrepancy, the function space $\mathcal{F}$ needs to have some constraints. Otherwise the discrepancy is trivial. 
A broad class of distances can be viewed as special cases of IPM. For instance, choosing all the functions whose integration under $q$ is zero yields Stein discrepancy \citep{gorham2017measuring}. 
%Other forms of IPM have been widely utilized in GAN literatures.

%Another weakness of SVGD lies in the asymptotic convergence with respect to KL divergence. On one hand, KL divergence is not symmetric and usually not stable for optimization due to its division format. In comparison, KSD and Fisher divergence are more robust. On the other hand, in absence of the normalizing constant, SVGD approximates the derivative of the KL divergence within an RKHS. Hence, it is hard to quantify how well is the KL divergence being minimized. In comparison, our proposed methods minimize over KSD and Fisher divergence, which themselves are interpretable and doesn't rely on the normalizing constant.

%But our goal here is sampling, that is to produce samples approximating a given density.
%On a completely different field, deep learning, Generative Adversarial Network (GAN) has achieved astonishing empirical success in generating life-like images. 

% The introduction of GAN has completely revolutionized how we model distribution of images. The high capacity of deep neural networks as generators, the flexibility of the training scheme, the huge amount of training data, all contribute to its success and the potential of GAN is unlimited. Recently \citet{brock2018large} bigGAN 
%Observing the similarities of GAN and sampling, a natural question to ask is what if instead of true samples, we train GAN with true density? 

\paragraph{Capacity of the Generator}
Deep neural networks as a function space has great flexibility and capacity. 
\citep{cybenko1989approximations} showed that even a single hidden layer can approximate continuous functions on compact subsets of $\RR^n$ arbitrarily well, as long as the number of neurons is large enough.
When modeling distributions, neural networks as generator has great capacity and can well approximate almost any distribution by transforming simple ones such as Gaussian or uniform distribution. 
\cite{lu2020universal} establishes a universal approximation theorem for deep neural networks for expressing distributions. 
When pushed through a sufficiently large neural network, even a one-dimensional distribution can be arbitrarily close to high-dimensional targets in Wasserstein distance \citep{yang2021capacity, perekrestenko2020constructive}.

Stein discrepancy can serve as a bridging tool between true samples and true density. This enables various frameworks in IPM-based GAN to be directly developed in parallel. Motivated by the connection, we propose two new sampling methods (Figure \ref{fig:parallel}).

\section{KSD neural sampler} % KSD NEURAL SAMPLER}
Let $q(x)$ denote the un-normalized target density with support on $\cX\subset\RR^{d}$ and $p_z(z)$ be the reference distribution that generates noises $z \in \RR^{d_0}$. 
%Denote the reference distribution by $p_z(\cdot)$ and the noise by $z$ with density $p_z(z)$ supported on $\RR^{d_0}$. 
Let $G_\theta$ denote our sampler, which is a multi-layer neural network parametrized by $\theta$. Let $p_\theta(x)$ be the underlying density of the generated samples $x = G_\theta(z)$. In summary, our setup is as follows:
\begin{align*}
z &\sim p_z(z),\quad G_\theta(z) = x \sim p_\theta(x)
\end{align*}
We want to train the network parameters $\theta$ so that $p_\theta(x)$ is a good approximation to the target $q(x)$.

\subsection{Methodology}
Evaluating the generated samples $X=\cbr{G_\theta(z_i)}_{i=1}^n$ is equivalent to conducting one-sample goodness-of-fit test. When $q(x)$ is un-normalized, one well-defined testing framework is based on kernelized Stein discrepancy. KSD is the counterpart of maximum mean discrepancy (MMD) in two-sample test \citep{gretton2012}. By choosing the IPM with $\cF$ being an unit ball in RKHS, \citet{li2015} proposed MMD-GAN, which simplifies the GAN framework by eliminating the need of training a discriminator network. As a result, MMD-GAN is more stable and easier to train. 

Motivated by MMD-GAN, we propose to train the generator $G_\theta$ by directly minimizing KSD with respect to $\theta$ using gradient-based optimization. At each iteration, a batch of samples $X$ %$
%$X$
%$\cbr{G_\theta(z_1), \cdots, G_\theta(z_n)}$ 
are generated by passing samples from the reference distribution, %$\cbr{z_1, \cdots, z_n}\sim p_z(z)$, 
$z_i\sim p_z$, through the generator $G_\theta$. 
%we sample a batch of noise $\cbr{z_1, \cdots, z_n}\sim p_z(z)$ and calculate the corresponding samples $\cbr{G_\theta(z_1), \cdots, G_\theta(z_n)}$. 
The empirical KSD is then calculated by plugging in the samples $X$ to formula (\ref{eqn:KSD_estimator}), %which measures how well our current samples are approximating $q(x)$. 
and $\theta$ is updated to the direction that minimizes the empirical KSD. Algorithm \ref{alg:ksd} summarizes our training procedure.
%Iteratively updating $\theta$ to minimize the empirical KSD until convergence. The loss in KSD-NS directly shows the sample quality. KSD is always non-negative and a smaller KSD indicates a better sample quality. 

\begin{algorithm}[!ht]
\caption{KSD-NS}
\begin{algorithmic}[1]
\STATE \textbf{Input:} (un-normalized) target density $q$, reference distribution $p_z$, number of iterations $N$, learning rate $\alpha$, mini-batch size $n$.
\STATE \textbf{Initialize} parameter $\theta$ for the generator network.
\STATE \textbf{For} iteration $t=1,\dots, N$, \textbf{Do}
\STATE \hspace{0.1in} Generate i.i.d. noise $z_1, \dots, z_n\sim p_z$;
\STATE \hspace{0.1in} Obtain fake sample $G_{\theta}(z_1), \dots, G_{\theta}(z_n)$;
\STATE \hspace{0.1in} Compute empirical $\widehat{\KSD}(p_{\theta},q)$;
\STATE \hspace{0.1in} Compute gradients $\nabla_{\theta} \widehat{\KSD}(p_{\theta},q)$;
\STATE \hspace{0.1in} update $\theta \leftarrow \theta - \alpha \nabla_{\theta} \widehat{\KSD}(p_{\theta},q)$;
\STATE \textbf{End For}
\end{algorithmic}\label{alg:ksd}
\end{algorithm}
%\paragraph{Remark} 
    %is always telling us about the sample quality. Small KSD indicates good sample quality.
    %When KSD is small, what we can tell is that in the support of the samples, the score function of $p_\theta$ matches the target $S_q$ well. An almost-zero KSD doesn't necessarily imply $p_\theta$ captures all the modes or recovers all the support of the true density.
%In the next section, we show that KSD-NS is theoretically sound. With sufficient batch size, empirical KSD loss converging to zero implies weak convergence of the sampling distribution. 

%{\bf Here, you can explain in more details how above two advantages of our method over Stein-GAN lead to empirical advantages...}

\subsection{Mini-batch error bound}\label{sec:ksd-err}
%\subsection{Bound for General Graph Kernel}

The optimization described in Algorithm \ref{alg:ksd} involves evaluating the expectation under $p_\theta$ and it is approximated by the mini-batch sample mean. Natural questions to ask include when the empirical KSD is minimized and what we can say about the population KSD. In the following, we demonstrate that the generalization error is bounded when mini-batch sample size is sufficiently large. 

Let $X_\theta = \{x_1,\cdots, x_n\}$ be one batch of generated samples from our generator $G_{\theta}(\cdot)$ with $\Theta$ being the parameter space. Denote $\hat{\theta}$ and $\theta^*$ as the values minimizing the empirical and population KSD,
\[
    \hat{\theta}  = \argmin_{\theta\in\Theta}\widehat{\KSD}(X_\theta, q), \quad
    \theta^*      = \argmin_{\theta\in\Theta}\KSD(p_{\theta}, q).
\] 
%We are interested in bounding the difference of $\KSD\rbr{p_{\hat{\theta}}, q} -  \KSD\rbr{p_{\theta^*}, q}$, whose upper bound is given in the following theorem. 
We show that the difference $\KSD\rbr{p_{\hat{\theta}}, q} -  \KSD\rbr{p_{\theta^*}, q}$ is upper bounded in the following theorem. 

\begin{theorem}
\label{thm:ksd}
Assume $q$ and $k(\cdot, \cdot)$ satisfy some smoothness conditions so that the newly defined kernel $u_q$ in (\ref{eqn:u_q}) is $L_1$-Lipschitz with one of the arguments fixed. Under some norm constraints on the weight matrix of each layer of the generator $G_\theta$, then for any $\epsilon>0$, 
% \[
%     \KSD\rbr{p_{\hat{\theta}}, q} \le \KSD\rbr{p_{\theta^*}, q} + \cO\rbr{\frac{C_d}{\sqrt{n}}} + 4\epsilon, 
% \]
% holds with probability at least $1-2\exp(-\epsilon^2n/2)$. And
\[
    \big|\hat{\KSD}\rbr{p_{\hat{\theta}}, q} - \KSD\rbr{p_{\theta^*}, q}\big| \le \cO\rbr{\frac{C_d}{\sqrt{n}}} + 5\epsilon, 
\]
holds with probability at least $1-3\exp(-\epsilon^2n/2)$
where $C_d$ is a function of the dimension $d$.  
\end{theorem}

\paragraph{Remark} The norm constraints on neural networks in theorem \ref{thm:ksd} require the norm of each weight matrix to be bounded. Commonly used norms are Frobenius norm, $W_{p,q}$ norm and other matrix norms \citep{golowich2017size}. More details are in the supplementary section \ref{a1}.

%%%%%%%%%%

 Theorem \ref{thm:ksd} implies that in practice, with enough batch size, the generator $G_\theta$ can be trusted if we observe a small KSD loss. However, we want to raise the following point. 
When KSD is small, it means that within the area of generated samples, the score function of $p_\theta$ matches the target score function $S_q$ well. An almost-zero empirical KSD does not necessarily imply capturing all the modes or recovering all the support of the true density. 
     %Admittedly, local trap is a common problem across various sampling methods, but our KSD-NS demonstrates strong resistance to this issue in simulations.

\subsection{Metrization of weak convergence}
 KSD based on commonly used kernels, such as Gaussian kernel, Matern kernel, fail to detect non-convergence when $d \ge 3$ \citep{gorham2017measuring}. 
 %However, KSD used in our neural sampler is exempt from such curse of dimensionality and we show that with some mild constraints, convergence to zero of KSD-NS does imply weak convergence of $p_\theta$ to $q$.
The issue of KSD with Gaussian kernel in higher dimensions can be traced back to the fast decaying kernel function. If we choose a heavy-tail kernel, such as Inverse Multi-Quadratic (IMQ) kernel, the corresponding KSD can detect non-convergence. The following theorem is from \citet{gorham2017measuring}.

\begin{theorem} Under IMQ kernel $k(x,y) = \rbr{c^2 + \norm{x-y}_2^2}^\beta$ where $c>0$ and $\beta \in (-1,0)$, $\KSD(p_{\theta},q) \rightarrow 0$ implies $p_{\theta} \overset{d}{\longrightarrow} q$.
\end{theorem}
%The above theorem shows that IMQ KSD detects non-convergence. Since IMQ is a bounded kernel, the corresponding KSD is well-defined as long as $F(p_{\theta},q) < \infty$. If we use Gaussian kernel or other popular kernels, we can still ensure weak convergence if we enforce uniformly tightness \citep{merolla2016deep}. One simple approach is through weight clipping of the generator, see appendix \ref{kconverge} for details. 

%Suppose k(x,y) is Gaussian kernel, and $P_{\theta}$ is uniformly tight.If $\KSD(p_{\theta},q) \rightarrow 0$, then $p_{\theta} \Rightarrow q$

By choosing the appropriate kernel, the KSD-NS is theoretically sound. However, in practice, the performance of our model might deteriorate as dimension goes higher. In the next section, we introduce the Fisher divergence neural sampler, which expands RKHS to $\cL_2$ space to have better discriminative power in higher dimensions.

\section{Fisher neural sampler} % FISHER NEURAL SAMPLER}
The ease of computation for kernel methods does not come free. 
RKHS is a relatively small function space and the expressive power decays when dimension goes higher. 
For instance, in generating images, empirical performance of MMD-GAN \citep{li2015} is usually not comparable to more computationally intensive GANs like Wasserstein GAN \citep{arjovsky2017wasserstein, gulrajani2017improved}. % A more common practice of MMD-GAN is to operate on extracted features of images, rather than the raw ones. 
In this section, we go beyond kernels and introduce a stronger divergence between distributions.

\subsection{Methodology}
Instead of an unit-ball in RKHS, we expand the function space $\cF$ in Stein discrepancy (\ref{eqn:ksdexp}) to be the $\cL_2$ space. Next, we approximate $\cL_2$ functions by another multi-layer neural network $\boldsymbol{f}_\eta(x)$ parametrized by $\eta$. (\ref{eqn:ksdexp}) becomes:
\[
    {\cD_\eta(p_\theta, q)} 
    = \sup_{\eta} \cbr{ \EE_{x\sim p_\theta}\tr\rbr{S_q(x) \boldsymbol{f}_\eta(x)^\top + \nabla_x \boldsymbol{f}_\eta(x)}}.
\]
Neural networks as functions are not square integrable by nature, since they don't vanish at infinity by default. To impose the $\cL_2$ constraint, we add an $\cL_2$ penalty term and thus our loss function becomes
\begin{align*}
    L_{\eta,\lambda}(p_\theta, q)=\mathcal{D}_\eta(p_\theta, q) - \lambda\,\EE_{x\sim p_\theta}\sbr{\boldsymbol{f}_\eta^\top\boldsymbol{f}_\eta},
\end{align*}
where $\lambda$ is a tuning parameter. Our training objective is
$$
    \min_{\theta}\max_{\eta}L_{\eta,\lambda}(p_\theta, q).  
$$ 
The \textit{ideal} training scheme is:
\begin{description}
\item[step 1] Initialize the generator $G_\theta$ and the discriminator $f_\eta$, both with \textit{infinite} capacity.
\item[step 2] Fix $\theta$, train $\eta$ to \textit{optimal}.
\item[step 3] Fix $\eta$, train $\theta$ for one step.
\item[step 3] Repeat step 2 and 3 until \textit{global} convergence.
\end{description}
The \textit{ideal} part mainly refers to training the discriminator to optimal and the discriminator itself has large enough capacity. The proposed training scheme is similar to that in Wasserstein GAN \citep{arjovsky2017wasserstein} and Fisher GAN \citep{mroueh2017fisher}. Under the optimality assumptions, next we show the extension from RKHS to $\cL_2$ indeed introduces a stronger convergence.

\subsection{Optimal discriminator} \label{sec:fisher_div}
The Fisher divergence between two densities $p$ and $q$ is defined as
\[
    \F(p\, ||\, q) = \EE_{x\sim p}\norm{\nabla_x \log(p) - \nabla_x \log(q)}^2_2.
\]
We now show that Fisher divergence is the corresponding loss of our ideal training scheme, provided that the discriminator network has enough capacity and is trained to global optimal. 
\begin{theorem}
\label{fisher}
The optimal discriminator function is 
$
    \frac{1}{2\lambda}\rbr{S_q(x) - S_p(x)}.
$ 
Training the generator with the optimal discriminator corresponds to minimizing the fisher divergence between $p_\theta$ and $q$. The corresponding optimal loss is
\[
    \frac{1}{4\lambda}  \EE_{x\sim p_\theta} \norm{S_q(x) - S_{p_\theta}(x)}_2^2.
\]
\end{theorem}

One observation is that when our sampling distribution $p_\theta$ is close to the target $q$, the discriminator function $f_\eta$ tends to zero. Naturally, $f_\eta$ can be used as an diagnostic tool to evaluate how well our neural sampler is working.

\paragraph{Fisher Divergence vs. KSD}
Fisher divergence dominates KSD in the following sense \citep{liu2016kernelized}:
$$
    \KSD(p, q) \le {\sqrt{\EE_{x, x'\sim p}[k(x, x')^2]}} \cdot \F(p||q).
$$
In fact, it is stronger than a lot of other distances between distributions, such as total variation, Hellinger distance, Wasserstein distance, etc \citep{ley2013stein}. 
%In addition, if $k(x, x')$ is positive definite and in the Stein class of $p$, and %$S_q-S_p\in\mathcal{H}^d$, we have for $p\ne q$,
%$$\sqrt{KSD}\ge F(p||q)/||S_q - S_q||_{\mathcal{H}^d}$$

\paragraph{Fisher Divergence vs. KL Divergence}
KL divergence is not symmetric and usually not stable for optimization due to its division format, while KSD and Fisher divergence are more robust. 
Under mild conditions, according to Sobolev inequality, Fisher divergence is a stronger distance than KL divergence.
%, which serves as the objective distance in SVGD. It implies that, theoretically, our framework has higher potentiality than SVGD.
Moreover, when the normalizing constant of the target density is unknown, the KL divergence can only be calculated up to additive constant. Thus, it is hard to quantify how well the KL divergence is being minimized. In comparison, both KSD and Fisher divergence only rely on the score function and hence, the values are directly interpretable as goodness-of-fit test statistics.

The optimality assumption on discriminator may seem unrealistic. However, 
\begin{itemize}
    \item Optimality of discriminator is an usual assumption for all GAN models mentioned in this paper. Optimization in deep neural networks are highly non-convex and the mini-max game in GAN model is extremely hard to characterize. Losing the assumption require tremendous amount of work \citep{arora2017generalization}.
    \item Many results suggest that deep neural networks with large capacity usually generalize well. Bad local minimum is scarce and more efficient optimization tools to escape saddle points are being developed \citep{kawaguchi2016deep, lecun2015deep, jin2017escape}.
\end{itemize}

In practice, we suggest choosing a large enough discriminator network and after each iteration of $\theta$, we train $\eta$ for multiple times. Algorithm \ref{alg:FD} summarizes our training procedure.

\begin{algorithm}[!ht]
\caption{Fisher-NS}
\begin{algorithmic}[1]
\STATE \textbf{Input:} un-normalized density $q(x)$, noise density $p_z(z)$, number of  step 2 iterations $m$, number for step 4 iterations $T$, tuning parameter $\lambda$, learning rate $\alpha_{1},\alpha_{2}$, mini-batch size $n$.
\STATE \textbf{Initialize} parameter $\theta$ and $\eta$ for both neural networks.
\STATE \textbf{For} iteration $t=1,\ldots, T$, \textbf{Do}
\STATE \hspace{0.1in} Generate i.i.d. noise inputs $z_1, \ldots, z_n$ from $p_z$;
\STATE \hspace{0.1in} Obtain fake sample $G_{\theta}(z_1), \cdots, G_{\theta}(z_n)$;
\STATE \hspace{0.1in} \textbf{For} $h = 1,\ldots, m$, \textbf{Do}:
\STATE \hspace{0.2in} Compute empirical loss $L_{\eta, \lambda}(p_{\theta},q)$;
\STATE \hspace{0.2in} Compute gradient $\triangledown_{\eta} L_{\eta, \lambda}(p_{\theta},q)$;
\STATE \hspace{0.2in} $\eta \leftarrow \eta + \alpha_{1} \triangledown_{\eta} L_{\eta, \lambda}(p_{\theta},q)$;
\STATE \hspace{0.1in} \textbf{End For}
\STATE \hspace{0.1in} Compute empirical loss $L_{\eta, \lambda}(p_{\theta},q)$;
\STATE \hspace{0.1in} Compute gradient $\triangledown_{\theta} L_{\eta, \lambda}(p_{\theta},q)$;
\STATE \hspace{0.1in} $\theta \leftarrow \theta - \alpha_{2} \triangledown_{\theta} L_{\eta, \lambda}(p_{\theta},q)$;
\STATE \textbf{End For}
\end{algorithmic}\label{alg:FD}
\end{algorithm}

\paragraph{Discriminator Initialization}
For more efficient training, we can initialize $\boldsymbol{f}_\eta$ around the optimal $\boldsymbol{f}^*$ from the KSD case. 
Fisher-NS is an extension from KSD-NS, where the discriminative function space is enlarged from RKHS to $\cL_2$ space. Let $\boldsymbol{f}^*_{\cH}$ be the optimal discriminative function in the RKHS. To make the training process of Fisher-NS more efficient, we can initialize the discriminative function around $\boldsymbol{f}^*_{\cH}$. Let $\boldsymbol{f}_\xi$ be a neural network function parametrized by $\xi$  and closely initialized around 0 and let $\boldsymbol{f} = \boldsymbol{f}_\xi + \boldsymbol{f}^*_{\cH}$ be the discriminative function to be optimized. Then the objective becomes
\begin{align}
\label{eqn:init}
      \sup_{\boldsymbol{f}_\xi\in\cL_2} \cbr{ \EE_{x\sim p_\theta}\tr\rbr{S_q(x) (\boldsymbol{f}_\xi + \boldsymbol{f}^*_{\cH})^\top + \nabla_x (\boldsymbol{f}_\xi + \boldsymbol{f}^*_{\cH}}}= {\cD_{\xi}(p_\theta, q)} + \KSD
\end{align}
Since all the operations of the discriminative function in the objective is linear, we can separate the objective into the initialization part and the KSD part. Therefore, initializing $\boldsymbol{f}_\eta$ to be around $\boldsymbol{f}^*_{\cH}$ is equivalent to adding $\KSD$ into the training objective and initialize the network close to zero. The KSD in (\ref{eqn:init}) can be thought as a regularization term and to make this formulation more flexible, we adapt our training objective to 
\[
\cD_{\xi, \lambda} = {\cD_{\xi}(p_\theta, q)} + \gamma\cdot \KSD
\]
In practice, we gradually decay $\gamma$ from 1 to 0 along the training process.

\section{Related work}
%\paragraph{SVGD}
%Given an initial distribution $p_0(x)$, the idea of SVGD is to learn a nonlinear transformation $T$ such that the distribution of $T(x)$ approximates the target distribution $q$ in the sense of  Kullback-Leibler (KL) divergence. However, directly learning the nonlinear transformation $T$ is difficult and SVGD circumvents this difficulty by constructing $T$ with incremental linear updates $x^\prime = x + \epsilon \boldsymbol{f}(x)$. The key observation is that
% \begin{align}
% \label{svgd}
% \nabla_{\epsilon} \mathrm{KL}(p'\,||\,q)\Bigg|_{\epsilon=0} = -\,\EE_{x\sim p}[\tr(S_q\boldsymbol{f}^\top(x) + \nabla_{x}\boldsymbol{f}(x))]
% \end{align}
% where $x\sim p$ and $x^\prime\sim p^\prime$. Confining $\boldsymbol{f}$ inside a unit ball of RKHS gives the optimal $\boldsymbol{f}^*$ in a closed form, in a similar spirit to KSD. 
%Running SVGD asymptotically minimizes the KL divergence between the sampling distribution and the target. 

%In practice, we found that SVGD converges relatively fast but suffers from local trap and high sensitivity to the location of initial particles. 

%However, one weakness of SVGD is that the nonlinear transformation $T$ can not be stored in any form after finishing an SVGD chain. As a consequence, we have to run extra SVGD chains when extra samples are needed. In comparison, our proposed neural sampler learns a preservable nonlinear transformation $T$ trained by neural networks. In this way, our framework is able to generate new samples without additional efforts once the transformation is learned. 

The fusion of deep learning and sampling is not new.
\citet{song2017nice} proposed A-NICE-MC, where the proposal distribution in MCMC is, instead of domain-agnostic, adversarially trained using neural networks. The authors show that A-NICE-MC is faster than Hamilton Monte Carlo (HMC) \citep{neal2011mcmc} in terms of effective sample size. Our neural sampler is fundamentally different from MCMC since we are training a preservable transformation. Once trained, we could generate independent samples instantaneously. 

In variational inference, \citet{rezende2015variational} greatly enhanced its flexibility by constructing the variational family through normalizing flow, where a simple initial density is transformed into a more complex one by a sequence of invertible transformations. The invertibility condition posts lots of constraints on the transformation and special forms must be taken \citep{dinh2014nice, dinh2016density}.
In contrast, our sampler is more flexible and the initial density does not have to be of the same dimension as the target. Another essential difference lies in the objective. Our samplers are trained with Stein discrepancy while variational inference often relies on KL divergence. %Different objectives in variational inference can lead to wildly different results. 
\citet{ranganath2016operator} proposed a more general variational operator which include Stein discrepancy as a special case. In comparison, we focus on realizing Stein discrepancy with different discriminative function spaces and develop specific algorithms. 

Another type of sampling methods that targets KL divergence are Stein variational gradient descent (SVGD) \citep{liu2016stein} and its network version, Stein GAN \citep{wang2016learning}. 
SVGD sequentially updates a set of particles to approximately minimize the KL divergence, while Stein GAN trains a neural network to sample from the target distribution by iteratively adjusting the weights according to the SVGD updates. 
%From a GAN perspective, in each iteration of Stein GAN, the discriminator is performing a two sample test between the currently generated samples and the one-step updated samples by SVGD. 
%This method generalized SVGD to training neural networks and it is minimizing the KL divergence between the sampling distribution and the target inside a RKHS.
Although Stein GAN shares a similar setup to our KSD-NS, they have completely different objectives. 
As discussed in section \ref{sec:fisher_div}, there are many advantages of KSD and Fisher divergence over KL divergence. 
Moreover, the objective in SVGD only approximately minimizes the KL divergence, by using a projected gradient of KL into a RKHS space. In comparison, the KSD gradient can be estimated unbiasedly.
As shown in section \ref{sec:ksd-err}, KSD-NS is theoretically sound -- with a sufficient batch size, empirical KSD loss converging to zero implies weak convergence of the sampling distribution. 

%our method enjoys many advantages. 
%While the loss of Stein GAN is not interpretable, the loss in our KSD-NS directly shows the sample quality -- a smaller KSD indicates a better sample quality. 
%is always telling us about the sample quality. Small KSD indicates good sample quality.
    %When KSD is small, what we can tell is that in the support of the samples, the score function of $p_\theta$ matches the target $S_q$ well. An almost-zero KSD doesn't necessarily imply $p_\theta$ captures all the modes or recovers all the support of the true density.

   %Stein GAN utilizes the derivative of KL divergence, which is more local comparing to KSD. 
   %Even though both methods used RKHS and kernel trick,
%Empirical results show that our KSD-NS tends to capture better global structures and less likely to drop mode. Stein GAN is more sensitive to initialization and suffers from local trap more severely.

KSD-NS is not the only method designed to directly minimize KSD. \citet{chen2018stein} proposed the Stein Points, a sequential sampling method that generates new point (sample) by minimizing the empirical KSD given all the previous points. However, finding the global optimal is challenging and it is not feasible in generating large amount of samples. 
%finding the global optimal is hard, especially in the high dimensional case, and it is not feasible to generate a large amount of samples. 
% Stein Points \citep{chen2018stein} is a sequential sampling method where the next sample is solved to achieve the global minimum of empirical KSD given all the previous points. The sample generated are essentially deterministic.  %
In comparison, our KSD-NS directly minimizes KSD by gradient-based algorithms on neural network weights which enables instantaneous sampling.

%%%%%%%%%%%%%%%%%%%%%%%%%%%%%%%%%%%%%%%%%%%%%%%%%%%%%%%%%%%%%%%%%%%%%%%%%%%%%%%%%%%%%%%%%%%%%%%%%%%%%%%%%%%%%%%%%%

\section{Experiments} % SIMULATIONS}
\label{sec:simulation}

% Our simulation results show that the model can learn the underlying distributions well. In comparison, our methods demonstrate superior ability to handle \textbf{multimodality} and avoid \textbf{local trap}
We evaluate our neural samplers on both toy examples and real world problems, and compare the performance with Stein GAN, SVGD and classic sampling methods, such as stochastic gradient Langevin dynamics (SGLD) \citep{WelingM2011bayesian} and HMC \citep{neal2011mcmc}. Results on 2-dimensional Gaussian mixtures showed that our methods are capable of capturing detailed local structures in the target distribution. Comparing with other benchmarking methods, our neural samplers also have superior ability to handle ``multimodality" and avoid ``local trap". Futhermode, when applied to high dimensional real world data, our methods also achieve better test accuracies. All the experiment details are stated in section \ref{simulation} of the supplementary material.

%%%%%%%%%%%%%%%%%%%%%%%%%%%%%%%%%%%%%%%%%%%%%%%%%%%%%%%%%%%%%%%%%%%%
% \twocolumn[
\begin{figure*}
    \centering
    \includegraphics[width=.8\textwidth]{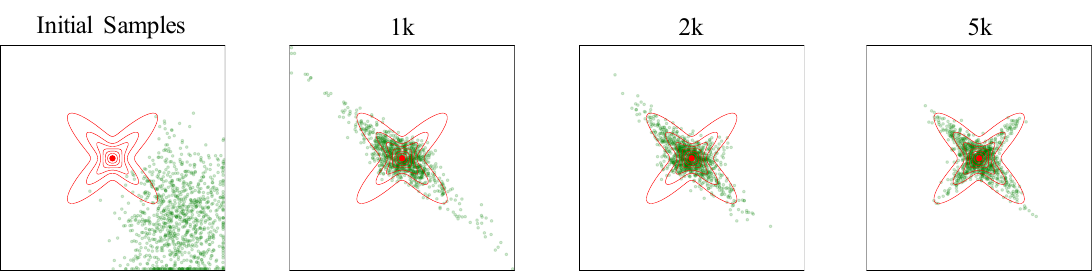}
    \caption{Toy example with 2D Gaussian mixture that trained with KSD-NS. The red contours represent the target distribution and the green dots are generated samples. From left to right are the initialization, 1000, 2000 and 5000 iterations correspondingly.}% The red dashed lines are the target density function and the blue shades are the histogram. From left to right are KSD-NS at 0, 3000, 6000, 9000 iterations respectively.}
    \label{fig:KSD_sgx}
\end{figure*}
% ]
%%%%%%%%%%%%%%%%%%%%%%%%%%%%%%%%%%%%%%%%%%%%%%

\paragraph{Unimodal Gaussian mixtures} 
The first toy example is a unimodal 2-dimensional Gaussian mixtures. The target distribution is %  to illustrate how our sampler transforms a reference distribution to match the target
$
    q(x) = 0.5\cdot \mathcal{N} \rbr{x;~ \bm{0},~I_2(0.8)} + 0.5\cdot \mathcal{N}\rbr{x;~ \bm{0},~I_2(-0.8)}
$ where ${I}_2(\rho)$ denotes the $2\times2$ matrix with 1 on the diagonal and $\rho$ as off-diagonal elements. Figure \ref{fig:KSD_sgx} shows how the sampling distribution evolves during the training with KSD-NS and Fisher-NS yields a similar result. This example shows that the detailed local structure of this target distribution is well captured by our methods. 

\paragraph{Multimodal Gaussian mixtures} 
In practice, sampling from multimodal distributions is usually very challenging, especially when the modes are far from each other. Even for GAN \citep{goodfellow2014}, a mixture of Gaussians with well separated modes could be hard when trained with true samples \citep{metz2017unrolled}. In this toy example, we set the target to be a equal mixture of 8 2-dimensional standard Gaussian components that are equally spaced on a circle of radius 15. Our neural samplers are compared with SVGD, Stein GAN and HMC on this example. To make the task more difficult, we set the initial stage for all the methods to be far away from any of the true modes. Figure \ref{fig:8Gaussian} shows the samples at different stage. For fair comparison, the same network configuration and initialization are shared between network based methods (Stein GAN and our KSD-NS and Fisher-NS), and SVGD particles are initialized with the same initial samples. Our neural samplers perform similarly in this case, and hence, we only show one trajectory of the training. SVGD and Stein GAN are experiencing a similar mode dropping issue, and thus, only SVGD samples are shown here. In the HMC case, the chain is initialized at the sample mean of the initial samples of other methods and after burn-in, 1000 consecutive samples are showed at each stage. It also experience a severe mode dropping problem. The result suggests that the our proposed methods are more powerful in exploring the global structures.

\begin{figure*}
    \centering
    \includegraphics[width=.9\textwidth]{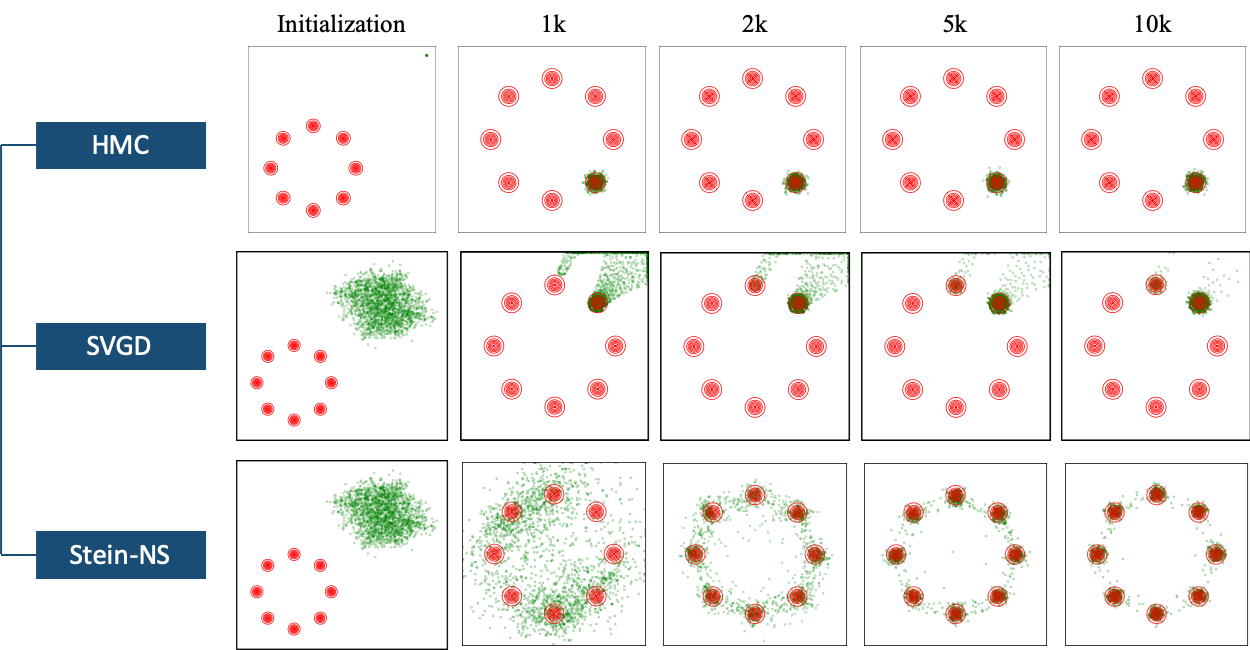}
    \caption{Toy example with multimodal 2D Gaussian mixtures. The red contour shows the target distribution and the green dots are generated sample(s) / particles at iteration 0, 1k, 2k, 5k and 10k respectively.  }
    \label{fig:8Gaussian}
\end{figure*}

%Given only the score function, we have no idea on the location information of the modes. Therefore, in the next example, in order to illustrate the capability of our methods to avoid local trap, we consider a 2-dimensional Gaussian mixture model with modes far from each other.
%Specifically, the target distribution $q$ is a mixture of 8 standard Gaussian components equally spaced on a circle of radius 15 with equal mixing weights. To make the task more difficult and more realistic, we choose initializations to be far away from the true modes.  

% In presence of multimodality, sampling methods should be able to consistently recover all the modes. We now investigate different sampling methods' ability to recover modes and avoid local traps in another Gaussian mixture set up. 

%Now we compare different sampling methods in this hard 1D Gaussian mixture case. Let the target density be $q(x) = 0.4\ \mathcal{N} (x; -10,1) + 0.2\ \mathcal{N}(x; 0,1) + 0.4\ \mathcal{N}(x; 10,1)$. Different sampling methods are used to calculate the mean and variance of the target. Results are summarized in Table \ref{table:1}.

\paragraph{Bayesian logistic regression on the Covertype data}
%We apply Stein neural sampler to Bayesian logistic regression model for binary classification and test our methods on the Covertype dataset. 
The covertype data set \citep{blackard1998forest} contains 581,012 observations of 54 features and a binary response. We use the same setting as \citet{liu2016stein}, where the prior of the weights is $p(w|\alpha) = \mathcal{N}(w;~0, \alpha^{-1})$ and $p(\alpha) = \rm{Gamma}(\alpha;~ 1, 0.01)$. The data set is randomly split into the training set (80\%) and testing set (20\%).
%and methods are evaluated with testing accuracy.
Our methods are compared with Stein GAN, SVGD, SGLD and doubly stochastic variational inference (DSVI) \citep{Titsias2014doubly} on this data set. For SGLD, DSVI and SVGD, the model is trained with 3 epoches of the training set (about 15k iterations), while for neural network based methods (Fisher-NS, KSD-NS and Stein GAN), we run them until convergence. Table \ref{Covertype} shows the mean and standard deviation of classification accuracies on the testing set with 30 replications of each method. The KSD-NS has a lower variances across different replications, but Fisher-NS achieves a higher accuracy on average. 

\begin{table}[h]
\caption{Test accuracies for Bayesian logistic regression on the Covertype dataset. }%The result showing below are the average test accuracies and ...}
\label{Covertype}
\vskip 0.1in
\begin{center}
\begin{small}
\begin{sc}
\begin{tabular}{ccc}
\toprule
SGLD                 & DSVI                 & SVGD                 \\
75.09\% $\pm$ 0.20\% & 73.46\% $\pm$ 4.52\% & 74.76\% $\pm$ 0.47\% \\
% 75.09\%              & 73.46\%              & 74.76                \\
% (0.20\%)             & (4.52\%)             & (0.47\%)             \\
\midrule
SteinGAN             & KSD-NS               & Fisher-NS            \\
75.37\% $\pm$ 0.19\% & 76.17\% $\pm$ 0.21\% & 76.22\% $\pm$ 0.43\% \\
% 75.31\%              & 76.09\%              & 75.49\%              \\
% (0.03\%)             & (0.12\%)             & (0.13\%)             \\
\bottomrule
\end{tabular}
\end{sc}
\end{small}
\end{center}
\vskip -0.1in
\end{table}

\section{Conclusion} %CONCLUSION}
In this paper, we propose two novel frameworks that directly learns preservable transformations from random noise to target distributions. KSD-NS enjoys theoretical guarantee and 
Fisher-NS further extends the discriminative function space from RKHS to $L_2$ space and optimally converges with respect to Fisher divergence. 
The introduction of GAN to sampling is exciting. Using Stein discrepancy as a bridge, numerous variants of GAN and their related techniques can be potentially applied in parallel to sampling.

\bibliography{stein_gan}
\bibliographystyle{abbrvnat}

%%%%%%%%%%%%%%%%%%%%%%%%%%%%%%%%%%%%%%%%%%%%%%%%%%%%%%%%%%%%%%%%%%%%%%%%%%%%%%%%%%%%%%%%%%%%%%%%%%%%%%%%%%%%%%%%%%%%%%%%%%%%%%%%%%%%%%%%%%%%%%%%%%%%%%%%%%%%%%%%%%%%%%%%%%%%%%%%%%%%%%%%%%%%%%%%

%%%%%%%%%%%%%%%%%%%%%%%%%%%%%%%%%%%%%%%%%%%%%%%%%%%%%%%%%%%%%%%%%%%%%%%%%%%%%%%%%%%%%%%%%%%%%%%%%%%%%%%%%%%%%%%%%%%%%%%%%%%%%%%%%%%%%%%%%%%%%%%%%%%%%%%%%%%%%%%%%%%%%%%%%%%%%%%%%%%%%%%%%%%%%%%%
\newpage
% \title{Supplementary Materials for Stein Neural Sampler}
% \author{Tianyang Hu\thanks{Department of Statistics, Purdue University} ,~
%     Zixiang Chen\thanks{Department of Statistics, Tsinghua University} ,~
%     Hanxi Sun\footnotemark[1] ,~
%     Jincheng Bai\footnotemark[1] ,~
%     Mao Ye\footnotemark[1] ,~
%     Guang Cheng\footnotemark[1]
% }
% \date{}
% \maketitle
% \onecolumn
% \section*{APPENDIX}
\section*{Supplementary Materials for Stein Neural Sampler}

\renewcommand\thesection{S}
\setcounter{subsection}{0}
%\addcontentsline{toc}{section}{Supplementary Materials}

\subsection{Proof of Theorem \ref{thm:ksd}}
\label{a1}
\begin{lemma}
\label{kernel}
(Theorem 3.7 from \cite{liu2016kernelized}) 
Assume $k(x, x')$ is a positive definite kernel in the Stein class of p, with positive eigenvalues {$\lambda_j$} and eigenfunctions {$e_j(x)$}, then $u_q(x, x')$ is also a positive definite kernel, and can be rewritten into
\begin{align}
\label{uq}
u_q(x,x') = \sum_j \lambda_j [\mathcal{A}_q e_j(x)]^\top [\mathcal{A}_q e_j(x')],
\end{align}
where $\mathcal{A}_q$ is the Stein operator acted on $e_j$ that 
\begin{align}
\mathcal{A}_q(f) = \nabla \log q(x)\cdot f + \nabla f.
\end{align}
\end{lemma}

\paragraph{Gaussian Kernel\citep{fasshauer2011positive}} Gaussian kernel is a popular characteristic kernel written as
$$k(x, x') = \exp(-\frac{||x-x'||^2}{2\sigma^2}).$$
Its eigenexpansion is 
\begin{align}
    \lambda_j &\propto b^j, \quad b<1,\\
    e_j(x) &\propto \frac{\exp(-a||x||^2)}{\sqrt{2^j j!}} \prod_{i=1}^d H_j(x_i\sqrt{2c}),
\end{align}
where $a,b,c>0$ are some constants depending on $\sigma$, and $H_k$ is $k$-th order Hermite polynomial. The eigenfunctions are 
$L2$-orthonorm. For details, please refer to section 6.2 of \citep{fasshauer2011positive}.

\begin{lemma}
(McDiarmid’s inequality, \citet{mendelson2003few}) Let $X_1,\cdots, X_n\in \cX$ be independent random variables and let $f:\cX^n\to\RR$ be a function of $X_1,\cdots, X_n$. Assume there exists $c_1,\cdots, c_n\ge 0$ such that $\forall i, x_1,\cdots, x_n, x_i^\prime\in \cX$,
$$|f(x_1,\cdots, x_i, \cdots, x_n)-f(x_1,\cdots, x_i^\prime, \cdots, x_n)|\le c_i.$$
Then, for all $\epsilon>0$,
$$\PP\left(f-\EE(f)\ge \epsilon\right) \le \exp\left(-\frac{2\epsilon}{\sum_{i=1}^n c_i^2}\right).$$

\end{lemma}

\begin{lemma}
\label{rc}
(Norm-based Sample Complexity Control (\cite{golowich2017size})) Let $\mathcal{H}_d$ be the class of real-valued neural networks of depth $D$ over domain $\mathcal{Z}$, where each weight matrix $W_j$ has Frobenius norm at most $M_F(j)$. Let the activation function be 1-Lipschitz, positive-homogeneous (such as the ReLU). 
Denote $\hat{\mathfrak{R}}_n (\mathcal{H})$ to be the empirical Rademacher complexity of $\mathcal{H}$. Then,
\begin{align*}
    \hat{\mathfrak{R}}_n(\mathcal{H}_d) \le \frac{B\left(\sqrt{2D\log 2}+1\right)\prod_{j=1}^D M_F(j)}{\sqrt{n}}.
\end{align*}
where $B>0$ is the range of the input distribution such that $||z||\le B$ almost surely. 
\end{lemma}

\begin{lemma}
\label{contraction}
(Extension of Ledoux-Talagrand contraction inequality \citep{van2016symmetrization})
Let $u:\mathbb{R}^d\to \mathbb{R}$ be L-Lipschitz functions w.r.t. $L_1$ norm, i.e. $\forall x,y\in\mathbb{R}^d, |u(x) - u(y)|\le L||x-y||_1$.  For some function space $\mathcal{F} = \{f=\left(g_1(x), \cdots, g_d(x)\right)^\top:\mathbb{R}^{d_0}\to \mathbb{R}^d\}$, denote
$\mathcal{F}_i = \{g_i(x):\mathbb{R}^{d_0}\to \mathbb{R}\}$ for $i=1,2, \cdots, d$, accordingly. Then,
\begin{align*}
    \hat{\mathfrak{R}}_n(u\circ \mathcal{F}) \le 2^{d-1} L\sum_{i=1}^d \hat{\mathfrak{R}}_n(\mathcal{F}_i),
\end{align*}
where $\circ$ means composition and $u\circ \mathcal{F} = \{u\circ f: f\in\mathcal{F}\}$. 
\end{lemma}

\paragraph{Theorem \ref{thm:ksd}}
Assume $q$ and $k(\cdot, \cdot)$ satisfy some smoothness conditions so that the newly defined kernel $u_q$ in lemma \ref{kernel} is $L_1$-Lipschitz with one of the argument fixed. If generator $G_\theta$ satisfy the conditions in \ref{rc}. Then,
For any $\epsilon>0$, with probability at least $\exp(-\epsilon^2n/2)$ the following bound holds,
\begin{equation}
 \textrm{KSD}(p_{\hat{\theta}}, q) \le \textrm{KSD}(p_{\theta^*}, q) + \mathcal{O}\left(\frac{2^{d}d B\sqrt{D}\prod_{j=1}^D M_F(j)}{\sqrt{n}}\right) + \epsilon. 
\label{eqn:GE}
\end{equation}
\begin{proof}
For the ease of notation, let's denote
$$\cE (\theta) = \widehat{\textrm{KSD}}(X_{\theta}, q),\quad \cT (\theta) =\textrm{KSD}(p_{\theta}, q).$$
By applying the large deviation bound on U-statistics of \citep{hoeffding1963probability}, we have that for any $\theta\in\Theta$
\begin{equation}
\mathbb P\left( |\cE (\theta)- \mathbb E(\cE (\theta))| >  \epsilon\right) \le 2\exp\left(-\frac{\epsilon^2n}{16}\right).
\label{eqn:T17}
\end{equation}
Note that (\ref{eqn:T17}) holds for any fixed $\theta$. Since $\theta^*$ is the population MMD minimizer that doesn't depend on samples, we have $\mathbb E(\cE (\theta^*))= \cT (\theta^*)$, which yields 
$$\mathbb P\left( |\cE (\theta^*)- \cT (\theta^*)| >  \epsilon\right) \le 2\exp\left(-\frac{\epsilon^2n}{16}\right).$$
On the other hand, $\hat{\theta}$ is the empirical MMD minimizer and to bound it, we want to show that for some $\beta_\epsilon$ s.t. 
\begin{align}
\label{eqn:la}
    \PP(\sup_\theta|\cE(\theta)-\cT(\theta)|> \epsilon)<\beta_\epsilon.
\end{align}
Apply (\ref{eqn:KSD_estimator}), we can write
\begin{align}
\label{Icase}
\sup_\theta|\cE(\theta)-\cT(\theta)| &\le \sup_\theta\bigg|\frac{1}{n(n-1)} \sum_{i=1}^n \sum_{j\ne i}^nu(X_i,X_{j})-\EE\left(u(X_i, X_j)\right)\bigg|\\
&:= \sup_\theta I(X_\theta). 
\end{align}
For $I(X_\theta)$, notice that the bounded condition for McDiarmid's inequality still holds that 
\begin{align*}
|\sup_\theta I(X_\theta)-\sup_\theta I(X_\theta')|\le\sup_\theta|I(X_\theta)-I(X_\theta')|\le\frac{2}{n}.
\end{align*}
where $X_\theta$ and $X_\theta'$ only differ in one element.
Then McDiarmid's inequality gives us that
\begin{align}
\label{EI2}
\PP\bigg(\sup_\theta I(X_\theta) - \EE(\sup_\theta I(X_\theta))>\epsilon\bigg) \le& \exp\left(-\frac{\epsilon^2n}{2}\right).
\end{align}
With high probability, $\sup_\theta I(X_\theta)$ can be bounded by $\EE(\sup_\theta I(X_\theta))+\epsilon$. Now we give a bound for $\EE(\sup_\theta I(X_\theta))$.
\begin{align*}
\label{marginal}
\EE\rbr{\sup_\theta I(X_\theta)} &= \EE_{p_\theta}\left(\sup_\theta\Big| \frac{1}{n(n-1)}\sum_{i=1}^n \sum_{j\ne i}^n u(X_i,X_{j})-\EE_{p_\theta}\left(u(X_i,X_j)\right)\Big|\right)\\
&= \EE_{p_\theta}\left(\sup_\theta\bigg| \frac{1}{n}\sum_{i=1}^n\frac{1}{n-1} \sum_{j\ne i}^n \Big(u(X_i,X_j)-\EE_{p_\theta}\left(u(X_i,X_j)\right)\Big)\bigg|\right)\\
&\le \frac{1}{n}\sum_{i=1}^n\EE_{p_\theta}\left(\sup_\theta\bigg|\frac{1}{n-1}\Big(  \sum_{j\ne i}^{n-1} u(X_i,X_{j})-\EE_{p_\theta}\left(u(X_i,X_j)\right)\Big)\bigg|\right)\\
&=\EE_{p_\theta}\left(\sup_\theta\bigg|\frac{1}{n-1}  \sum_{j= 1}^{n-1}u(X_n,X_{j})-\EE_{p_\theta}\left(u(X_n,X_j)\right)\bigg|\right)\\
&=\EE_{X_n}\left[\EE_{p_\theta}\bigg(\sup_\theta\Big|\frac{1}{n-1}  \sum_{j= 1}^{n-1}u(X_n,X_{j})-\EE_{p_\theta}\left(u(X_n,X_j)\right)\Big|\bigg)\Bigg|X_n\right].
%&\le& 2\mathfrak{R}_{n-1}\big(\cF_\theta\big)
\end{align*}
Once $X_n$ is fixed, $u(X_n,X_{j})$ for different $j$'s are independent. By standard argument of Rademacher complexity, we have
\begin{eqnarray}
\label{rcbound}
\EE_{p_\theta}\bigg(\sup_\theta\Big|\frac{1}{n-1}  \sum_{j= 1}^{n-1}k(X_n,X_{j})-\EE_{p_\theta}\left(u(X_n,X_j)\right)\Big|\bigg)\Bigg|X_n&\le& 2\mathfrak{R}_{n-1}\Big(\cF_{\theta, X_n}\Big),
\end{eqnarray}
 
where $$\cF_{\theta, X_n} = \{u(X_n,G_\theta(z)): z \sim p_z\textrm{~and independent of~}X_n\}.$$

Combine the assumption of $u_q$ being Lipschitz with lemma \ref{contraction}, we have 
\begin{align}
    \mathfrak{R}_{n}\Big(\cF_{\theta, X_n}\Big) \le 2^{d-1}\sum_{i=1}^d \mathfrak{R}_{n}(\cF_{\theta, X_n, i}).
\end{align}

Applying lemma \ref{rc} yields
\begin{align}
    2\mathfrak{R}_{n}\Big(\cF_{\theta, X_n}\Big) &\le  \frac{2^{d}dL\cdot B\left(\sqrt{2D\log 2}+1\right)\prod_{j=1}^D M_F(j)}{\sqrt{n}} :=\eta(n, d)\\
    &= \mathcal{O}\left(\frac{2^{d}dB\sqrt{D}\prod_{j=1}^D M_F(j)}{\sqrt{n}}\right). 
\end{align}

Together with (\ref{EI2}) and (\ref{rcbound}), we can further get
\begin{align*}
\label{I2}
&\PP\bigg(\sup_\theta I(X_\theta) >\eta(n, d) +\epsilon\bigg) \\
\le&\PP\bigg(\sup_\theta I(X_\theta) > 2\mathfrak{R}_{n-1}\Big(\cF_{\theta, X_n}\Big) + \epsilon\bigg) \\
\le& \PP\bigg(\sup_\theta I(X_\theta) > \EE(\sup_\theta I(Y_\theta))+\epsilon\bigg) \le \exp\left(-\frac{\epsilon^2n}{2}\right).
\end{align*}

Now we can get 
\begin{align*}
    \PP\left(\cT(\hat{\theta}) - \cT(\theta^*)>4\epsilon+ \eta\right) &= \PP\left(\cT(\hat{\theta}) - \cE(\hat{\theta}) + \cE(\hat{\theta}) - \cT(\theta^*)> 4\epsilon+ \eta\right)\\
    &\le \mathbb P\left(\cT(\hat{\theta}) - \cE(\hat{\theta}) + \cE(\theta^*) - \cT(\theta^*)> 4\epsilon+ \eta\right)\\
    &\le \mathbb P\left(|\cE(\hat{\theta}) - \cT(\hat{\theta})|>\epsilon+ \eta\right) +\mathbb P\left(|\cE(\theta^{*}) - \cT(\theta^{*})|>2\sqrt{2}\ \epsilon\right) \\
    &\le 2\exp\left(-\frac{\epsilon^2n}{2}\right). 
\end{align*}
Together with (\ref{eqn:la}), the theorem is proved and the bound goes to zero if $\epsilon^2n$ go to infinity.

Additionally, we can easily get
\begin{align*}
    \PP\left(\big|\cE(\hat{\theta}) - \cT(\theta^*)\big|>5\epsilon+ 2\eta\right) &= \PP\left(\big|\cE(\hat{\theta}) - \cT(\hat{\theta}) + \cT(\hat{\theta}) - \cT(\theta^*)\big|> 5\epsilon+ 2\eta\right)\\
    &\le \mathbb P\left(|\cE(\hat{\theta}) - \cT(\hat{\theta})|>\epsilon+ \eta\right) +\mathbb P\left(|\cT(\theta^{*}) - \cT(\theta^{*})|>4 \epsilon + \eta \right) \\
    &\le 3\exp\left(-\frac{\epsilon^2n}{2}\right). 
\end{align*}
\end{proof}

\paragraph{Remark} The Lipschitz condition for kernel $u_q$ is not hard to satisfy. From Lemma \ref{kernel}, if we use Gaussian kernel, as long as $S_q$ doesn't have exponential tails, the Lipschitz condition is satisfied.

\noindent
In our application, we can choose a wide range of noise distributions as long as it is easy to sample and regular enough. If we choose uniform distribution, then $B = \mathcal{O}(\sqrt{d})$. If assume $M_F(j)\le M\le\infty$ for any $j = 1,2 ,\cdots, D$. Then (\ref{eqn:GE}) becomes
\begin{equation*}
\textrm{KSD}(p_{\hat{\theta}}, q) \le \textrm{KSD}(p_{\theta^*}, q) + \mathcal{O}\left(\frac{2^{d}d^{3/2}}{\sqrt{n}}\right) + \epsilon. 
\end{equation*}

\subsection{Proof of Theorem \ref{fisher}}
\begin{lemma}
$$\EE_{x\sim p} \tr\left(\nabla_x \log q(x) \boldsymbol{f}(x)^\top  + \nabla_x \boldsymbol{f}(x)\right) =  \EE_{x\sim p} \tr\left((\nabla_x \log q(x) - \nabla_x \log p(x)) \boldsymbol{f}(x)^\top  \right).$$
\end{lemma}

\begin{proof}
\begin{align*}
\EE_{x\sim p} \tr\left(\nabla_x \log p(x) \boldsymbol{f}(x)^\top  + \nabla_x \boldsymbol{f}(x)\right) &=0 \\ 
\Rightarrow \EE_{x\sim p} \tr\left(\nabla_x \boldsymbol{f}(x)\right) &= - \EE_{x\sim p} \tr\left(\nabla_x \log p(x)  \boldsymbol{f}(x)^\top \right).
\end{align*}
\end{proof}

\begin{lemma}
$$ |\EE_{x\sim p} \tr\left(\boldsymbol{g}(x)  \boldsymbol{f}(x)^\top  \right)| \leq \sqrt{\EE_{x\sim p} \tr\left(\boldsymbol{g}(x) ^\top \boldsymbol{g}(x)  \right)*\EE_{x\sim p} \tr\left(\boldsymbol{f}(x) ^\top \boldsymbol{f}(x)  \right)}.$$ 

The equality holds iff $f  \propto g\ $ a.s. $x \sim p$.
\end{lemma}

\begin{proof}
Firstly, we have $\tr(\boldsymbol{f}(x)\boldsymbol{f}(x)^\top)=\tr(\boldsymbol{f}(x)^\top\boldsymbol{f}(x)) = \boldsymbol{f}(x)^\top\boldsymbol{f}(x) = \|\boldsymbol{f}(x)\|_{2}^2$
\begin{align*}
\EE_{x\sim p} \tr\left(\boldsymbol{(f- t* g)}  \boldsymbol{(f-t* g)}^\top  \right)  &\geq 0 \\
\Rightarrow \EE_{x\sim p} \tr\left(\boldsymbol{f} ^\top \boldsymbol{f}  \right) + t^2\EE_{x\sim p} \tr\left(\boldsymbol{g} ^\top \boldsymbol{g}  \right)  &\geq  2t\EE_{x\sim p} \tr\left(\boldsymbol{g}\boldsymbol{f} ^\top   \right).
\end{align*}

Because inequality hold for all t, so the lemma is proved.
\end{proof}

\paragraph{Theorem \ref{fisher}}
The optimum discriminator is $$\frac{1}{2\lambda}(S_q - S_p).$$
Training generator equals minimize the fisher divergence of p and q $$\frac{1}{4\lambda}  \EE_{x\sim p} \tr\left((S_q - S_p)^\top (S_q - S_p) \right).$$

\begin{proof}
Let our loss function be $L$. Because $\tr(\boldsymbol{f}(x)\boldsymbol{f}(x)^\top)=\tr(\boldsymbol{f}(x)^\top\boldsymbol{f}(x)) = \boldsymbol{f}(x)^\top\boldsymbol{f}(x) = \|\boldsymbol{f}(x)\|_{2}^2$. Then we have
\begin{align}
L &=  \EE_{x\sim p} \tr\left(S_q(x) \boldsymbol{f}(x)^\top  + \nabla_x \boldsymbol{f}(x)\right)  - \lambda \EE_{x\sim p}[\tr(\boldsymbol{f}(x)^\top \boldsymbol{f}(x)) ] \\
&= \EE_{x\sim p} \tr\left((S_q(x) - S
_p(x)) \boldsymbol{f}(x)^\top  \right) - \lambda \EE_{x\sim p}[\tr(\boldsymbol{f}(x)^\top \boldsymbol{f}(x)) ] \\
&\leq \sqrt{\EE_{x\sim p} \tr\left((S_q - S_p)^\top (S_q - S_p) \right)\cdot \EE_{x\sim p} \tr\left(\boldsymbol{f}(x) ^\top \boldsymbol{f}(x)  \right)} - \lambda \EE_{x\sim p}\tr\big(\boldsymbol{f}(x)^\top \boldsymbol{f}(x)\big)\label{ie1} \\
&\leq \frac{1}{4\lambda}  \EE_{x\sim p} \tr\left((S_q - S_p)^\top (S_q - S_p) \right). \label{ie2}
\end{align}

Equality sign in (\ref{ie1}) holds iff $f \propto S_q - S_p, \ a.s \ \ x \sim p$. 

Equality sign in (\ref{ie2}) holds iff $$\EE_{x\sim p } \tr\left(\boldsymbol{f}(x) ^\top \boldsymbol{f}(x)  \right) = \frac{1}{4\lambda^2} \EE_{x\sim p} \tr\left((S_q-S_p)^\top (S_q-S_p) \right).$$

So argmax of $L$ is $\frac{1}{2\lambda}(S_q - S_p )\ \ a.s \ \  x \sim p$.
\end{proof}

\subsection{More About Fisher-NS}
\label{trainingG}
% \paragraph{Discriminator Initialization}
% Fisher-NS is an extension from KSD-NS, where the discriminative function space is enlarged from RKHS to $\cL_2$ space. Let $\boldsymbol{f}^*_{\cH}$ be the optimal discriminative function in the RKHS. To make the training process of Fisher-NS more efficient, we can initialize the discriminative function around $\boldsymbol{f}^*_{\cH}$. Let $\boldsymbol{f}_\xi$ be a neural network function parametrized by $\xi$  and closely initialized around 0 and let $\boldsymbol{f} = \boldsymbol{f}_\xi + \boldsymbol{f}^*_{\cH}$ be the discriminative function to be optimized. Then the objective becomes
% \begin{equation}
% \label{eqn:init}
%         \sup_{\boldsymbol{f}_\xi\in\cL_2} \cbr{ \EE_{x\sim p_\theta}\tr\rbr{S_q(x) (\boldsymbol{f}_\xi + \boldsymbol{f}^*_{\cH})^\top + \nabla_x (\boldsymbol{f}_\xi + \boldsymbol{f}^*_{\cH}}}
%     = {\cD_{\xi}(p_\theta, q)} + \KSD
% \end{equation}
% Since all the operations of the discriminative function in the objective is linear, we can separate the objective into the initialization part and the KSD part. Therefore, initializing $\boldsymbol{f}_\eta$ to be around $\boldsymbol{f}^*_{\cH}$ is equivalent to adding $\KSD$ into the training objective and initialize the network close to zero. The KSD in (\ref{eqn:init}) can be thought as a regularization term and to make this formulation more flexible, we adapt our training objective to 
% \[
% \cD_{\xi, \lambda} = {\cD_{\xi}(p_\theta, q)} + \gamma\cdot \KSD
% \]
% In practice, we gradually decay $\gamma$ from 1 to 0 along the training process. 

\paragraph{Training the Generator}
After the training cycle for the discriminator, we fix $\eta$ and train the generator $G_\theta$. Denote the loss function to be $L(\theta)$ and ideally, we would want $L(\theta)$ to be continuous with respect to $\theta$. Wasserstein GAN \citep{arjovsky2017wasserstein} gives a very intuitive explanation of the importance of this continuity. We now give some sufficient conditions, under which our training scheme satisfies the continuity condition with respect to $\theta$ for any discriminator function $f_\eta$.

\begin{theorem}
\label{thm:conti}
If the following conditions are satisfied:  1) both the generator's weights and the noises are bounded; 2) discriminator uses smooth activate function i.e. tanh, sigmoid, etc.; 3) target score function $s_q$ is continuously differentiable. 
Then $L(\theta)$ is continuous everywhere and differentiable almost everywhere w.r.t. $\theta$. 
\end{theorem}

%Suppose the following conditions are satisfied: 1. Generator with weight clipping and uniform noise, 2. discriminator use tanh as a activate function.3. Score function is continuously differentiable.

%Then $L(\theta)$ is continuous everywhere and differential almost everywhere w.r.t $\theta$. 

\begin{proof}
\label{thm:Lsmooth}
Using generator with weight clipping and uniform noise, we have a transform function $G_{\theta}$ which is lipschitz. As we can see later in Theorem A.10, there exist a compact set $\Omega$. $P(G_{\theta}(z) \in \Omega)$ = 1, $\forall \theta$. 

From the condition of discriminator we know that $f$ is smooth. So $\mathcal{A}_{q}f$ is continuously differentiable on $\Omega$, $f^{T}f$ is continuously differentiable on $\Omega$. So we have $\|\mathcal{A}_{q}f\|_{lip, \Omega} + \|f^{T}f\|_{lip, \Omega} < \infty$ and 
\begin{align*}
&\EE_{p_{\theta}}(\mathcal{A}_{q}f(x) - \lambda f(x)^{T}f(x)) - \EE_{p_{\theta'}}(\mathcal{A}_{q}f(x) - \lambda f(x)^{T}f(x))  \\
=& \big(\EE_{z}(\mathcal{A}_{q}f(G_{\theta}(z)) - \lambda f(G_{\theta}(z))^{T}f(G_{\theta}(z))) - \EE_{z}(\mathcal{A}_{q}f(G_{\theta'}(z)) - \lambda f(G_{\theta'}(z))^{T}f(G_{\theta}(z)')) \big) \\
\leq&(\|\mathcal{A}_{q}f\|_{lip, \Omega} +\lambda \|f^{T}f\|_{lip, \Omega})\EE_{z}(\|G_{\theta} - G_{\theta'}\|).
\end{align*}
Because z and $\theta$ all bounded, We know that G is locally lipschitz. For a given pair $(\theta,z)$ there is a constant $C(\theta,z)$ and an open set $U_{\theta}$ such that for every $(\theta', z) \in U_{\theta}$ we have
$$\|G_{\theta}(z)-G_{\theta'}(z)\| \leq C(\theta,z)\|\theta - \theta'\|.$$
Under the condition mentioned before, $E_{z}|C(\theta,z)| < \infty$, so we achieve
$$\|L(\theta)-L(\theta')\| \leq (\|\mathcal{A}_{q}f\|_{lip, \Omega} +\lambda \|f^{T}f\|_{lip, \Omega})\EE_{z}|C(\theta,z)|\|\theta-\theta'\|.$$
Therefore, $L(\theta)$ is locally Lipschitz and continuous everywhere. Lastly, applying Radamacher’s theorem proves $L(\theta)$ is differentiable almost everywhere, which completes the proof.
\end{proof}

\paragraph{Remark} These conditions are to impose some Lipschitz continuity. The first condition is trivially satisfied if we choose uniform as random noise and apply weight clipping to the generator. Except for $\theta$ being bounded, the other conditions are mild. It is true that procedures like weight clipping will make the function space smaller. But we can make the clipping range large enough to reach a fixed accuracy \citep{merolla2016deep}. The empirical difference should be negligible if the range is sufficiently large.

\subsection{Relationship to Wasserstein GAN}
Denote $\phi = \tr(\mathcal{A}_{q}f)$, then $\phi = \mathbf{div} (qf)/q$ and our loss function without penalty can be re-written as $E_{p}(\phi) - E_{q}(\phi)$.

\begin{lemma}
\citep{duran2010solutions} If $\Omega$ is a John domain, for any $v \in L_{0}^{l}(\Omega), l>1$ there exists $u \in \mathcal{W}_{0}^{1,l}$ such that $\mathbf{div} \ u = v$ in $\Omega$
\end{lemma}

In Wasserstein GAN, if we constrain the functions to be compactly supported and the expectation under the target distribution $\EE_q(f)$ to be zero, the result doesn't change.

\begin{theorem}
Suppose there exists $l > 1$ s.t $\|xq\|,\|q\| \in L_{0}^{l}$, then if we constrain $\phi = \tr(\mathcal{A}_{q}f)$ to be Lip-1 and compacted supported. Then the optimal loss function is Wasserstein-1 distance.
\end{theorem}

\begin{proof}
For every function $\phi$ which is  Lip-1 and has compact support, $\|\phi q\| \leq \|xq\| + c\cdot\|q\|$, where $c>0$ is some constant. So the equation has a solution, there exist $f$ which has a compact support s.t  $\phi = tr(\mathcal{A}_{q}f)$. 
\end{proof}

\paragraph{Remark}
Firstly, $\|xq\|, \|q\| \in L_{0}^{l}$ is extremely weak even for Cauchy distribution this condition holds. Secondly, we can apply weight clipping to $f$ to ensure $f$ has a compact support and $\tr(\mathcal{A}_{p}f)$ is Lip-1. 

\subsection{Weak convergence}\label{kconverge}
\begin{theorem}
If kernel $k(x, y)$ if bounded by constant $c$. Then $S(p_{\theta}, q) \leq c\cdot F(p_{\theta},q)$
\end{theorem}
\begin{proof}
\begin{align*}
S(p_{\theta}, q)^2 &= |\EE_{x,x'}((S_{p_{\theta}}(x)- S_{q}(x))^{T}k(x,x')(S_{p_{\theta}}(x')- S_{q}(x')))|^2    \\
&\leq \EE_{x,x'}(k(x,x')^2)\cdot \EE_{x,x'}(|(S_{p_{\theta}}(x)- S_{q}(x))^{T}(S_{p_{\theta}}(x')- S_{q}(x'))|^2)  \\
&\leq \EE_{x,x'}(k(x,x')^2)\cdot \EE_{x,x'}(\|(S_{p_{\theta}}(x)- S_{q}(x))\|_{2}^2\|(S_{p_{\theta}}(x')- S_{q}(x'))\|_{2}^2) \\
&=\EE_{x,x'}(k(x,x')^2)\cdot F(p_{\theta},q)^2 \\
&\leq c^2 \cdot F(p_{\theta},q)^2. 
\end{align*}
\end{proof}

\begin{theorem}
Suppose we use uniform or Gaussian noise, tanh or relu activate function for generator. Then $p_{\theta}$ is uniformly tight, if we clip the weight to $(-c,c)$ for any $c>0$.
\end{theorem}
\begin{proof}
Denote the transform function of generator is $G_{\theta}(x)$. Fix $z_{0}$ in the space. Then we know that there exist R, s.t  $\|G_{\theta}(z_{0})\| < R$ for all $\theta$. 
In addition, $G_{\theta}$ is a lipschitz function because the weight is clipped to $(-c,c)$. So there exist k s.t 
$\|G_{\theta}(x) - G_{\theta}(y)\| \leq k\|x-y\|$. So we have

$$P(\|G_{\theta}(z)\| > A) \leq P(\|G_{\theta}(z) - G_{\theta}(z_{0})\| > A - R) \leq P(\|z - z_{0}\| > (A - R)/k).$$  

Notice that z $\sim$ normal or uniform. For all $\epsilon >0$,  there exist $\hat{A}$ s.t $(\|z - z_{0}\| > \hat{A}/k) <\epsilon$.Therefore $P(\|G_{\theta}(z)\| > \hat{A}+R) < \epsilon$ holds for all $\theta$, which means $G_{\theta}$ are uniformly tight.
Moreover if noise is uniform, there exist  $\hat{A}$ s.t $P(\|G_{\theta}(z)\| > \hat{A}+R) = 0$ for all $\theta$.

\end{proof}

\subsection{Simulation Details}
\label{simulation}
\paragraph{Experiment Setting in Gaussian Mixtures}
Network based methods (KSD-NS, Fisher-NS and Stein GAN) share the same configurations. The generator/sampler is a plain network with $\tanh$ activation and two hidden layers of width 200. The discriminator network in Fisher-NS is of the same structure. The input noise (reference distribution) is chosen to be i.i.d. Normal(0, 10). The optimization is done in TensorFlow via RMSProp with 1e-3 learning rate for the generator. The discriminator in Fisher-NS is trained with gradient penalty. 

SVGD is trained with a step size of 0.3 (other step sizes shares a similar results). HMC is trained with the initial step size of 1 and 3 leapfrog steps. The first 5000 iterations are discarded (burn-in). 

\paragraph{Experiment Setting in Bayesian Logistic Regression}
%For Stein GAN and KSD-NS the setting are the same as before. For Fisher-NS, we apply a training technique where we 
Across all the methods, we use a mini-batch of 100 data points for each iteration (for each stage in DSVI).
The setting for SVGD is the same as in \citep{liu2016stein}.  All networks in this case are chosen to be 3-layer fully connected with $\tanh$ as the activation function. 
The learning rate of SGLD is chosen to be $0.1 / (t + 1)^{0.55}$ as suggested in \cite{WelingM2011bayesian}, and the average of the last 100 points is used for evaluation. For DSVI, the learning rate is $1e-07$ and 100 iterations is used for each stage. For SVGD, we use RBF kernel with bandwidth $h$ calculated by the "median trick" as in  \cite{liu2016stein}, and 100 particles is used for evaluation with step size being 0.05. For Fisher NS, the learning rate is 0.0001 for the discriminator and 0.0001. The $\cL_2$ constrain on the discriminator is imposed by the augmented Lagrangian as in \citet{mroueh2017fisher}. The optimization is done in TensorFlow via RMSProp. To reach convergence, KSD-NS takes more iterations (200k) compared to Fisher-NS and Stein GAN (50k). 

% Codes available at https://github.com/HanxiSun/SteinNS
\end{document}